\documentclass{article}

\usepackage{arxiv}

\usepackage[utf8]{inputenc} % allow utf-8 input
\usepackage[T1]{fontenc}    % use 8-bit T1 fonts
\usepackage{hyperref}       % hyperlinks
\usepackage{url}            % simple URL typesetting
\usepackage{booktabs}       % professional-quality tables
\usepackage{amsfonts}       % blackboard math symbols
\usepackage{nicefrac}       % compact symbols for 1/2, etc.
\usepackage{microtype}      % microtypography
\usepackage{lipsum}		% Can be removed after putting your text content
\usepackage{graphicx}
\usepackage{natbib}
\usepackage{doi}

\usepackage{mymacros}

\title{Performance Improvement Bounds for Lipschitz Configurable Markov Decision Processes}

\author{Alberto Maria Metelli \\
	Dipartimento di Elettronica Informazione e Bioingegneria\\ Politacnico di Milano\\ 32, Piazza Leonardo da Vinci\\ 20133, Milan, Italy\\
	\texttt{albertomaria.metelli@polimi.it} \\
}

\renewcommand{\shorttitle}{Performance Improvement Bounds for Lipschitz Conf-MDPs}

\hypersetup{
pdftitle={Performance Improvement Bounds for Lipschitz Configurable Markov Decision Processes},
pdfsubject={},
pdfauthor={Alberto Maria Metelli},
pdfkeywords={},
}

\begin{document}
\maketitle

\begin{abstract}
Configurable Markov Decision Processes (Conf-MDPs) have recently been introduced as an extension of the traditional Markov Decision Processes (MDPs) to model the real-world scenarios in which there is the possibility to intervene in the environment in order to \emph{configure} some of its parameters. In this paper, we focus on a particular subclass of Conf-MDP that satisfies regularity conditions, namely Lipschitz continuity. We start by providing a bound on the Wasserstein distance between $\gamma$-discounted stationary distributions induced by changing policy and configuration. This result generalizes the already existing bounds both for Conf-MDPs and traditional MDPs. Then, we derive a novel performance improvement lower bound.
\end{abstract}

\section{Introduction}
The framework of the \emph{Configurable Markov Decision Processes}~\cite[Conf-MDPs,][]{metelli2018configurable,metelli2019reinforcement,MetelliMR22} has been introduced in recent years to model a wide range of real-world scenarios in which an agent has the opportunity to alter some environmental parameters in order to improve its learning experience. Conf-MDPs can be thought to as an extension of the traditional Markov Decision Processes~\cite[MDP,][]{Puterman94} to account for scenarios that emerge quite often in the Reinforcement Learning~\cite[RL,][]{sutton2018reinforcement} problems, in which the environment rarely represents an immutable entity and can, indeed, be subject to partial control. In the Conf-MDP framework, the activity of altering the environmental parameters is named \emph{environment configuration} and serves different purposes. 

In the simplest scenario, the configuration is carried out by the agent itself that acts as a configurator. This might suggest, at a first sight, that environment configuration can be modeled within the agent actuation. While in principle this approach is possible, it tends to disregard the domain peculiarities of environment configuration that usually is performed at a slower frequency compared to policy learning and might generate notable costs. The more realistic setting is the one in which an additional entity, the \emph{configurator} is present and in charge of the environment configuration process. The configurator acts on the environment transition model (that we will call \emph{configuration} in this setting) and might have an objective that is different from that of the agent. Thus, we can distinguish between two scenarios~\cite{phdthesis, metelli2022overview}: the \emph{cooperative} and the \emph{non-cooperative} settings. In the former, the agent and the configurator share the same interests, \ie the have the same reward function. In such a case, the learning problem is simple as there is no conflict between the two involved entities and several algorithms have been proposed in the literature~\cite{metelli2018configurable, metelli2019reinforcement}. In the non-cooperative setting, instead, agent and configurator have possibly diverging interests, \ie they reward function differ. Here the definition of a solution concept is more challenging and requires considering game-theoretic equilibria~\cite{phdthesis}. Recently a regret minimization approach has been proposed to address the learning problem in general-sum non-cooperative Conf-MDPs~\cite{ramponi2021learning}.

In this paper, we focus on cooperative Conf-MDP\footnote{For the sake of brevity, we will simply use the abbreviation \quotes{Conf-MDP} to denote a cooperative Conf-MDP.} and we derive several theoretical results on the performance improvement that can be obtained as an effect of altering either the policy of the environment configuration. We start from the theoretical results presented in~\cite{metelli2018configurable}, and further developed in~\cite{phdthesis}, and we generalize them to a wide class of Conf-MDPs. Specifically, we focus on two kind of results: (i) bounds on the distance between the $\gamma$-discounted stationary distributions and (ii) performance improvement lower bounds. The former quantifies the distance in the state visitation distributions induced when changing simultaneously the agent policy and the environment configuration; whereas the latter provides a computable approximated expression of the minimum performance attained with the selection of a new policy and a new configuration. As demonstrated by previous works, lower bounds on the performance improvement represent valuable tools to build a wide variety of effective learning approaches, including \emph{safe}~\cite{PirottaRPC13, MetelliPCR21} and \emph{trust-region}~\cite{SchulmanLAJM15,AchiamHTA17} approaches.

More in detail, we focus on a specific class of regular Conf-MDPs that we name \emph{Lipschitz} (LC) Conf-MDPs, in analogy with the traditional Lipschitz MDPs~\cite{RachelsonL10, PirottaRB15}. We introduce them in Section~\ref{sec:smooth}. From an intuitive perspective, and with little inaccuracy, we can think of the LC condition as the existence of the first derivative of the relevant quantities. Then, in Section~\ref{sec:boundGamma}, we provide the bounds on the $\gamma$-discounted stationary distributions. We start with a bound that merges together the effects of the policy and the configuration and then we move to a looser result that highlights the individual contributions of the policy and of the configuration. We compare the derived bound with the ones already present in the literature (even for MDPs only) and we show that, several of them, can be derived by our result, under specific conditions. In Section~\ref{sec:FINITEadv}, we revise the notion of \emph{advantage function} for Conf-MDPs, which quantifies the one-step gain in performance obtained by either modifying the policy, the configuration, or both. We report the policy, configuration, and joint advantages and we study their relationships. Finally, in Section~\ref{sec:boundPI}, we provide the performance improvement lower bounds. Specifically, we first derive a general tight bound that is hardly computable as it involves quantities that are usually unknown in practice. Then, we move to the derivation of a looser bound that has the advantage of being computable. Although mainly theoretical, the contributions provided in this paper are of independent interest, even outside the Conf-MDP field,\footnote{Every result we present can be employed for traditional MPDs as well, by just assuming not to change the environment configuration.} and can be effectively employed as a starting point for the design of safe and trust region methods. Part of the results presented in this paper has previously appeared in~\cite{metelli2018configurable, phdthesis}.

\section{Preliminaries}
In this section, we introduce the necessary background about probability, Lipschitz continuity, Wasserstien distance (Section~\ref{sec:prelProb}) and Conf-MDPs (Section~\ref{sec:prelCMDP}), that will be employed in the following sections.

\subsection{Mathematical Background}\label{sec:prelProb}

\paragraph{Probability} 
Let $\Xs$ be a set and $\mathfrak{F}$ be a $\sigma$-algebra over $\Xs$. We denote with $\PM{\Xs}$ the set of probability measures over the measurable space $(\Xs, \mathfrak{F})$. Let $\Ys$ be a set, we denote with $\Delta_{\Ys}^{\Xs}$ the set of functions with signature $\Ys \rightarrow \PM{\Xs}$ and with $\Ys^{\Xs}$ the set of functions with signature $\Xs \rightarrow \Ys$. Let $x \in \Xs$, the Dirac delta measure is denoted by $\delta_{x}$. 

\paragraph{Lipschitz Continuity}
Let $(\Xs, d_{\Xs})$ and $(\Ys, d_{\Ys})$ be two metric spaces, where $d_{\Xs}: \Xs \rightarrow [0,+\infty)$ and $d_{\Ys}: \Ys \rightarrow [0,+\infty)$ are the corresponding distance functions. A function $f \in \Ys^{\Xs}$ is \emph{Lipschitz continuous} with Lipschitz constant $L_f$ ($L_f$-LC) if it holds that:
\begin{align}\label{eq:Lip}
	d_{\Ys}\left( f(x), f(\overline{x}) \right) \le L_f d_{\Xs} \left(x, \overline{x} \right), \qquad \forall x,\overline{x} \in \Xs.
\end{align}
We define the \emph{Lipschitz semi-norm} of function $f$ as the smallest $L_f > 0$ for which Equation~\eqref{eq:Lip} holds:
\begin{align*}
	\|f\|_L = \sup_{x,\overline{x} \in \Xs, \, x \neq \overline{x}} \frac{d_{\Ys}\left( f(x), f(\overline{x}) \right)}{d_{\Xs} \left(x, \overline{x} \right)}.
\end{align*}
Being a semi-norm, $\|\cdot\|_L$ is non-negative and fulfills the triangular inequality. If $\Ys = \Reals$, \ie $f$ is a real-valued function, we typically employ as distance function $d_{\Ys}$ the Euclidean distance, \ie $d_{\Ys}(y,\overline{y}) = |y- \overline{y}|$. If $\Ys = \Delta^{\mathcal{Z}}$, \ie $f$ has values in probability distributions,  we employ as distance function the Wasserstein distance ~\citep{villani2009optimal}.

\paragraph{Wasserstein Distance}
Let $p,q \in \PM{\Xs}$ be two probability measures over the metric space $(\Xs, d_{\Xs})$. The $L_1$-Wasserstein (or Kantorovich) distance between $p$ and $q$ is defined as~\cite{villani2009optimal}:
\begin{align*}
	\mathcal{W}(p,q) = \suplip \left| \int_{\Xs} \left(p(\de x)- q(\de x) \right) f(x) \right|.
\end{align*}
If the metric $d_{\Xs}$ is chosen to be the \emph{discrete metric}, \ie $d_{\Xs}(x,x') = \mathds{1}\{x \neq x'\}$ for all $x,x' \in \Xs$, then the $L_1$-Wasserstein reduces to the Total Variation (TV) divergence.

\subsection{Configurable Markov Decision Processes}\label{sec:prelCMDP}
A (cooperative) Configurable Markov Decision Process~\citep[Conf-MDP,][]{metelli2018configurable} is a tuple $\mathcal{C} = (\Ss,\As,r,\gamma)$, where $\Ss$ is the state space, $\As$ is the action space, $r \in \Reals^{\SASs}$ is the reward function mapping every state-action-next-state triple $(s,a,s') \in \SASs$ to the reward $r(s,a,s') \in \Reals$, and $\gamma \in [0,1)$ is the discount factor. The behavior of the configurator is modeled by a configuration (or transition model) $p \in \PM{\Ss}_{\SAs}$ that maps every state-action pair $(s,a) \in \SAs$ to the probability measure $p(\cdot|s,a) \in \PM{\Ss}$ of the next state. The behavior of the agent is modeled by a policy $\pi \in \PM{\As}_{\Ss}$ that  maps every state $s\in\Ss$ to the probability measure $\pi(\cdot|s) \in \PM{\As}$ of the action to be played. We denote with $p_\pi \in \PM{\Ss}_{\Ss}$ the state-state transition kernel, defined as:
\begin{align*}
	p_\pi(\de s'|s) = \int_{\As} \pi(\de a|s) p(\de s'|s,a), \qquad \forall s,s' \in \Ss.
\end{align*}

\paragraph{Value Functions}
Similarly to traditional MDPs, the \emph{value functions}~\cite{sutton2018reinforcement} can be define for Conf-MDPs as done in~\cite{phdthesis}. Let $(\pi,p) \in \PM{\As}_{\Ss} \times \PM{\Ss}_{\SAs}$ be a policy-configuration pair, we define the following {value functions}, for every $(s,a,s') \in \mathcal{S \times A\times S}$:
\begin{align*}
	&\VpiP(s) = \int_{\As} \pi(\de a|s) \int_{\Ss} p(\de s'|s,a) \left( r(s,a,s') + \gamma \VpiP(s') \right), \\
	&\QpiP(s,a) = \int_{\Ss} p(\de s'|s,a) \left( r(s,a,s') + \gamma \VpiP(s') \right),\\
	& \UpiP(s,a,s') =  r(s,a,s') + \gamma \VpiP(s').
\end{align*}
These value functions represent the \emph{expected discounted cumulative reward} (\ie the \emph{expected return}) experienced by the agent-configurator when interacting with the environment with a policy $\pi $ and a configuration $p$. While $\VpiP(s)$ (\ie state value function or V-function) and $\QpiP(s,a)$ (\ie state-action value function or Q-function) are defined analogously to the case of MDPs, $\UpiP(s,a,s')$ (\ie state-action-next-state value function or U-function) is a peculiar value function for Conf-MDPs~\cite{metelli2018configurable} that turns out to be relevant for learning optimal configurations. We refer the reader to~\cite[][Chapter 4]{phdthesis} for a detailed review of the value function and the Bellman operators and equations for Conf-MDPs.

\paragraph{$\gamma$-discounted Stationary Distributions} 
A policy-configuration pair $(\pi,p) \in \PM{\As}_{\Ss} \times \PM{\Ss}_{\SAs}$ induces a visitation distribution of the states. The $\gamma$-discounted stationary distribution accounts the expected discounted number of visits\footnote{\quotes{Discounted} means that if a visit to a given state is performed at step $t$ of interaction it counts $\gamma^t$.} induced by a policy-configuration pair $(\pi,p) \in \PM{\As}_{\Ss} \times \PM{\Ss}_{\SAs}$, starting from an initial state sampled from $\isd \in \PM{\Ss}$, and it is defined in several equivalent forms~\cite{sutton1999policy}:
\begin{align}\label{eq:discDistr}
	\piPD^{\gamma,\isd} = (1-\gamma) \isd + \gamma \piP \piPD^{\gamma,\isd} = (1-\gamma) \sum_{t = 0}^{+\infty} \gamma^t \isd p_\pi^t = (1-\gamma) \isd \left(\Id_{\Ss} - \gamma p_\pi \right)^{-1}.
\end{align}
For the sake of brevity, whenever clear from the context, we omit the superscripts $\gamma,\isd$, simply employing $\piPD$.

Given an initial state distribution $\isd \in \PM{\Ss}$, we define the \emph{expected return} of a policy-configuration pair $(\pi,p) \in \PM{\As}_{\Ss} \times \PM{\Ss}_{\SAs}$ in two interchangeable ways:
\begin{align*}
	J_{\pi,p} = \int_{\Ss} \isd(\de s) \VpiP(s) = \frac{1}{1-\gamma} \int_{\SASs} \piPD(\de s) \pi(\de a |s) p(\de s'|s,a) r(s,a,s').
\end{align*}
%
%\paragraph{Optimality} Given an initial state distribution $\isd \in \PM{\Ss}$, we define the \emph{expected return} of a policy-configuration pair $(\pi,p) \in \PM{\As}_{\Ss} \times \PM{\Ss}_{\SAs}$ in two interchangeable ways:
%\begin{align*}
%	J_{\pi,p} = \int_{\Ss} \isd(\de s) \VpiP(s) = \frac{1}{1-\gamma} \int_{\SASs} \piPD(\de s) \pi(\de a |s) p(\de s'|s,a) r(s,a,s').
%\end{align*}
%The joint goal of the agent and of the configurator consists in optimizing the expected return, \ie finding an optimal policy-configuration pair~\cite{metelli2018configurable}:
%\begin{align*}
%	(\pi^*,p^*) \in \argmax_{\pi \in \Pi, p \in \mathcal{P}} J_{\pi,p},
%\end{align*}
%where $\Pi \subseteq \PM{\As}_{\Ss}$ and $\mathcal{P} \subseteq \PM{\As}_{\SAs}$ are suitable policy and configuration spaces.

\section{Lipschitz Configurable Markov Decision Processes}\label{sec:smooth}
In this section, we introduce the notion of regular Conf-MDPs that we will employ in the derivation of the theoretical results. We rephrase the traditional notion of Lipschitz (LC) MDP~\cite{RachelsonL10, PirottaRB15} to the case of Conf-MDPs.

\begin{ass}[Lipschitz Configurable Markov Decision Processes (LC)]\label{def:lc}
Let $\mathcal{C} = (\Ss, \As, r, \gamma)$ be a Conf-MDP, where $(\Ss, d_{\Ss})$ and $(\As, d_{\As})$ are metric spaces endowed with the corresponding distance functions $d_{\Ss}$ and $d_{\As}$, respectively.  $\mathcal{C}$ is $L_r$-LC if it holds that:
\begin{align*}
	\left|r(s,a,s') - r(\overline{s}, \overline{a}, \overline{s}') \right| \le L_r d_{\SASs}\left((s,a,s'),(\overline{s},\overline{a},\overline{s}')\right), \qquad \forall (s,a,s'), (\overline{s}, \overline{a}, \overline{s}') \in \SAs,
\end{align*}
where $d_{\SASs}\left((s,a,s'),(\overline{s},\overline{a},\overline{s}')\right) \coloneqq d_{\Ss}(s,\overline{s}) + d_{\As}(a,\overline{a}) + d_{\Ss}(s',\overline{s}')$.
\end{ass}

Compared with~\cite{RachelsonL10, PirottaRB15}, we enforce a condition on the reward function only, since the transition model (\ie configuration) does not belong to the definition of the Conf-MDP. Moreover, with negligible loss of generality, as done in~\cite{RachelsonL10, PirottaRB15}, we consider a disjoint metric for the state and action spaces. We now proceed at introducing the notion of Lipschitz configuration and policy.

\begin{ass}[Lipschitz Configuration and Policy (LC)]\label{def:lppi}
Let $\mathcal{C} = (\Ss, \As, r, \gamma)$ be a Conf-MDP, where $(\Ss, d_{\Ss})$ and $(\As, d_{\As})$ are metric spaces endowed with the corresponding distance functions $d_{\Ss}$ and $d_{\As}$, respectively.
	Let $p \in \PPP$ be a configuration. $p$ is $L_p$-LC if it holds that:
	\begin{align*}
	\mathcal{W}\left(p(\cdot|s,a),p(\cdot|\overline{s}, \overline{a})\right) \le L_p d_{\SAs} \left((s,a),(\overline{s},\overline{a})\right), \quad \forall (s,a),(\overline{s},\overline{a}) \in \SAs,
	\end{align*}
	where $d_{\SAs} \left((s,a),(\overline{s},\overline{a})\right) \coloneqq  d_{\Ss}(s,\overline{s}) + d_{\As}(a,\overline{a}) $. 	Let $\pi \in \PP$ be a policy. $\pi$ is $L_\pi$-LC if it holds that:
	\begin{align*}
			\mathcal{W}\left(\pi(\cdot|s),\pi(\cdot|\overline{s} )\right) \le L_\pi  d_{\Ss}(s,\overline{s}), \quad \forall s,\overline{s} \in \Ss.
	\end{align*}
\end{ass}

The choice of employing the Wasserstein distance over other distributional divergences, such as the TV divergence~\cite{MunosS08}, is justified by the fact that the former allows quantifying the distance between deterministic distributions. Instead, the TV divergence takes its maximum value whenever the involved distributions have disjoint support. The Wasserstein distance is a standard choice in the RL literature~\cite{PirottaRB15, MetelliMBSR20}.

\paragraph{Lipschitz semi-norms of the Value functions}
Once we enforce the LC conditions on the Conf-MDP and on the policy and configurations (Assumptions~\ref{def:lc} and~\ref{def:lppi}), it is well-known that the state value function $\VpiP$ and state-action value function $\QpiP$ are LC, under the assumption that $\gamma L_p(1+L_\pi) < 1$~\cite{RachelsonL10}. In particular, their Lipschitz semi-norms (\ie Lipschitz constants) are bounded as:
\begin{align*}
	 \lip{\QpiP} \le \frac{L_r}{1-\gamma L_p (1+L_\pi)}, \qquad \lip{\VpiP} \le \lip{\QpiP} (1+L_{\pi}) \le \frac{L_r(1+L_\pi)}{1-\gamma L_p (1+L_\pi)}.
\end{align*}

The following result provides a bound on the Lipschitz semi-norm of the state-action-next-state value function $\UpiP$, that has been specifically introduced for the Conf-MDPs.

\begin{lemma}\label{lemma:ULip}
Let $\mathcal{C}$ be an $L_r$-LC Conf-MDP, $p \in \PPP$ be an $L_p$-LC configuration, and $\pi \in \PP$ be an $L_\pi$-LC policy. Then, the state-action-next-state value function $\UpiP$ is LC, under the assumption that $\gamma L_p(1+L_\pi) < 1$, with Lipschitz semi-norm:
\begin{align*}
	\lip{\UpiP} \le  L_r + \gamma \lip{\VpiP} \le \frac{L_r(2 + L_{\pi} - \gamma L_p(1+L_\pi))}{1-\gamma L_p(1+L_\pi)}.
\end{align*}
\end{lemma}

\begin{proof}
	Let $(s,a,s'),(\overline{s},\overline{a},\overline{s}') \in \SASs$, we have:
	\begin{align*}
		\left|\UpiP(s,a,s') - \UpiP(\overline{s},\overline{a},\overline{s}') \right| & = \left|r(s,a,s') + \gamma \VpiP(s') - r(\overline{s},\overline{a},\overline{s}') - \gamma \VpiP(\overline{s}')\right| \\
		& \le \left|r(s,a,s') - r(\overline{s},\overline{a},\overline{s}') \right| + \gamma \left| \VpiP(s') -  \VpiP(\overline{s}')\right| \\
		& \le  L_r d_{\SASs}\left((s,a,s'),(\overline{s},\overline{a},\overline{s}')\right)+ \gamma \lip{\VpiP} d_{\Ss}(s',\overline{s}') \\
		& \le \left( L_r + \gamma \lip{\VpiP} \right) d_{\SASs}((s,a,s'),(\overline{s},\overline{a},\overline{s}'),
	\end{align*}
	having observed that $d_{\SASs}\left((s,a,s'),(\overline{s},\overline{a},\overline{s}')\right) \ge  d_{\Ss}(s',\overline{s}')$.
\end{proof}

%
%To this end, we introduce the following notation. Let $f \in \Ys^{\Xs}$ be a function:
%\begin{align*}
%	L_f(x) \coloneqq \limsup_{\overline{x} \rightarrow x}\frac{d_\Ys(f(x),f(\overline{x})) }{d_{\Xs}(x,\overline{x})}
%\end{align*}
%
%
%\begin{ass}[Second-Order Smooth Configurable Markov Decision Processes]\label{def:lc}
%Let $\mathcal{C} = (\Ss, \As, r, \gamma)$ be a Conf-MDP, where $(\Ss, d_{\Ss})$ and $(\As, d_{\As})$ are metric spaces endowed with the corresponding distance functions $d_{\Ss}$ and $d_{\As}$, respectively. $\mathcal{C}$ is $L_r$-FOS if it holds that:
%\begin{align*}
%&	\left|L_r(s,a,s') - L_r(\overline{s}, \overline{a}, \overline{s}') \right| \le \beta_r \left(d_{\Ss}(s,\overline{s}) + d_{\As}(a,\overline{a}) + d_{\Ss}(s',\overline{s}')\right), \quad \forall (s,a,s'),(\overline{s},\overline{a},\overline{s}') \in \SASs.
%\end{align*}
%\end{ass}
%
%
%\begin{ass}[Second-Order Smooth Configuration and Policy]\label{def:lppi}
%	Let $p \in \PPP$ be a configuration. $p$ is $L_p$-FOS if it holds that:
%	\begin{align*}
%	\mathcal{W}\left(p(\cdot|s,a),p(\cdot|\overline{s}, \overline{a})\right) \le L_p \left( d_{\Ss}(s,\overline{s}) + d_{\As}(a,\overline{a}) \right), \quad \forall (s,a),(\overline{s},\overline{a}) \in \SAs.
%	\end{align*}
%	Let $\pi \in \PP$ be a policy. $\pi$ is $L_\pi$-FOS if it holds that:
%	\begin{align*}
%			\mathcal{W}\left(\pi(\cdot|s),\pi(\cdot|\overline{s} )\right) \le L_\pi  d_{\Ss}(s,\overline{s}), \quad \forall s,\overline{s} \in \Ss.
%	\end{align*}
%\end{ass}
%
%
%
%\begin{align*}
%	\wass{\frac{\pi'(}{2}}{\frac{}{2}}
%\end{align*}
%

\section{Bound on the $\gamma$-discounted Stationary Distribution}\label{sec:boundGamma}

In this section, we provide a bound for the Wasserstein distance of $\gamma$-discounted stationary distributions under different policy-configuration pairs. We start with a result that bounds the  Wasserstein distance  of the $\gamma$-discounted stationary distributions in terms of the Wasserstein distance between the corresponding state-state transition kernels (Section~\ref{sec:coupledGamma}) and, then, we provide a looser bound that decouples the contribution of the policy from that of the transition model (Section~\ref{sec:uncoupledGamma}). Both results are obtained for LC Conf-MDPs, policies, and configurations. Finally, we provide a comparison with similar results already present in the literature (Section~\ref{sec:comp}).

\subsection{Coupled Bound}\label{sec:coupledGamma}
The following theorem provides the coupled bound on the $\gamma$-discounted stationary distribution, that merges together the contributions of the configuration and of the policy.

\begin{restatable}[]{thr}{distributionsBoundCoupled}\label{thr:distributionsBoundCoupled}
	Let $\ConfMDP$ be a Conf-MDP, let $(\pi,p),(\pi',p') \in \PP \times \PPP$ be two policy-configuration pairs such that $\pi'$ is $L_{\pi'}$-LC and $p' $ is $L_{p'}$-LC. Then, if $\gamma L_{p'}(1+L_{\pi'}) < 1$, it holds that:
    \begin{equation*}
    	\mathcal{W}\left( \piPD[\pi'][p'], \piPD\right) \le  \frac{\gamma}{1-\gamma L_{p'}(1+L_{\pi'})}  \int_{\Ss}\piPD(\de s)  \mathcal{W}{\left( \piP[\pi'][p'](\cdot|s), \piP[\pi][p](\cdot|s) \right) }.
    \end{equation*} 
\end{restatable} 

\begin{proof}
	The derivation presents similarities with the proof of Lemma 2 of~\cite{PirottaRB15}. Exploiting the recursive equation of the $\gamma$-discounted state distribution (Equation~\ref{eq:discDistr}), we can write the distributions difference as follows, in operator form:
	\begin{align}
	\piPD[\pi'][p'] - \piPD & = (1-\gamma)\isd + \gamma \piPD[\pi'][p'] \piP[\pi'][p'] - (1-\gamma)\isd - \gamma \piPD \piP \notag \\
	& = \gamma \piPD[\pi'][p']\piP[\pi'][p'] - \gamma \piPD \piP \pm \piPD \piP[\pi'][p']  \notag \\
	& = \gamma  \left(\piPD[\pi'][p'] - \piPD \right)\piP[\pi'][p'] + \gamma \piPD \left( \piP[\pi'][p'] - \piP\right) \notag \\
	& = \gamma \piPD \left( \piP[\pi'][p'] - \piP\right) \left( \Id_{\Ss} - \gamma \piP[\pi'][p'] \right)^{-1},\label{p:0-1}
	\end{align}
	where we exploited the recursive definition of $\piPD[\pi'][p'] - \piPD$ and recalled that $\gamma < 1$. We proceed by computing the Wasserstein distance:
	\begin{align}
	\wass{\piPD[\pi'][p']}{\piPD} & = \suplip \left|\int_{\Ss} \left( \piPD[\pi'][p'](\de s) - \piPD(\de s) \right) f(s) \right|  \notag \\
	& = \gamma \suplip \left|\int_{\Ss} \piPD(\de s'') \int_{\Ss} \left( \piP[\pi'][p'](\de s'|s''), \piP[\pi][p](\de s'|s'') \right) \underbrace{\int_{\Ss} \left( \Id_{\Ss} - \gamma \piP[\pi'][p'] \right)^{-1} (\de s|s') f(s)}_{\eqqcolon g_f(s')} \right|\label{p:000} \\
	& \le  \gamma \int_{\Ss} \piPD(\de s'')  \suplip \left|\int_{\Ss} \left( \piP[\pi'][p'](\de s'|s''), \piP[\pi][p](\de s'|s'') \right) g_f(s') \right|\label{p:001} \\
	& \le \gamma \int_{\Ss} \piPD(\de s'')  \wass{ \piP[\pi'][p'](\cdot|s'')}{\piP[\pi][p](\cdot|s'') } \suplip \lip{g_f},\label{p:002}
	\end{align}
	where in line~\eqref{p:000} we exploited Equation~\eqref{p:0-1}, line~\eqref{p:001} follows from Jensen's inequality, and line~\eqref{p:002} comes from the definition of Wasserstein distance. We now compute the Lipschitz semi-norm $\lip{g_f}$. Let us introduce the distribution $\mu_{\pi',p'}^{\gamma,\delta_{s}} = (1-\gamma) \delta_s + \gamma \piP[\pi'][p'] \mu_{\pi',p'}^{\gamma,\delta_{s}}= (1-\gamma) \delta_s \left( \Id_{\Ss} - \gamma \piP[\pi'][p']\right)^{-1}$, \ie the $\gamma$-discounted stationary distribution when the initial state distribution is the Dirac delta $\delta_s$. We observe that $(1-\gamma) g_f(s) =  \int_{\Ss} \mu_{\pi',p'}^{\gamma,\delta_{s}} (\de s') f(s')$. Thus, to compute the Lipschitz semi-norm $\lip{g_f}$, we proceed as follows for $s,\overline{s} \in \Ss$:
	\begin{align*}
		(1-\gamma) g_f(s) & - (1-\gamma) g_f(\overline{s})  = \int_{\Ss} \left( \mu_{\pi',p'}^{\gamma,\delta_{s}} (\de s') - \mu_{\pi',p'}^{\gamma,\delta_{\overline{s}}} (\de s')  \right) f(s') \le \wass{\mu_{\pi',p'}^{\gamma,\delta_{s}} }{\mu_{\pi',p'}^{\gamma,\delta_{\overline{s}}} }.
	\end{align*}
	Thus, we have:
	\begin{align}
		& \wass{\mu_{\pi',p'}^{\gamma,\delta_{s}} }{\mu_{\pi',p'}^{\gamma,\delta_{\overline{s}}} } = \suplip \Bigg|(1-\gamma) \int_{\Ss} \left( \delta_s(\de s') - \delta_{\overline{s}}(\de s')\right) f(s') \notag \\
		& \qquad\qquad\qquad \qquad\qquad\qquad\quad + \gamma \int_{\Ss} \left( \left( \piP[\pi'][p'] \mu_{\pi',p'}^{\gamma,\delta_{s}} \right) (\de s') -\left( \piP[\pi'][p'] \mu_{\pi',p'}^{\gamma,\delta_{\overline{s}}} \right)(\de s') \right) f(s') \Bigg| \label{p:10001} \\
		& \qquad \le (1-\gamma) d_{\Ss}(s,\overline{s}) + \gamma  \suplip \left|\int_{\Ss}\left( \mu_{\pi',p'}^{\gamma,\delta_{s}}(\de s'') - \mu_{\pi',p'}^{\gamma,\delta_{\overline{s}}}(\de s'') \right) \underbrace{ \int_{\Ss} \piP[\pi'][p'] (\de s'|s'') f(s')}_{\eqqcolon h_f(s'')} \right|\label{p:10002} \\
		& \qquad \le (1-\gamma) d_{\Ss}(s,\overline{s}) + \gamma \suplip  \lip{h_f} \suplip \left| \int_{\Ss}\left( \mu_{\pi',p'}^{\gamma,\delta_{s}}(\de s'') - \mu_{\pi',p'}^{\gamma,\delta_{\overline{s}}}(\de s'') \right)  f(s'')\right| \notag\\
		& \qquad = (1-\gamma) d_{\Ss}(s,\overline{s}) + \gamma \suplip  \lip{h_f} \wass{\mu_{\pi',p'}^{\gamma,\delta_{s}} }{\mu_{\pi',p'}^{\gamma,\delta_{\overline{s}}} },  \notag
	\end{align}
	where line~\eqref{p:10001} follows from the definition of $\mu_{\pi',p'}^{\gamma,\delta_{s}}$, line~\eqref{p:10002} follows from observing that $ \int_{\Ss} \left( \delta_s(\de s') - \delta_{\overline{s}}(\de s')\right) f(s') = f(s)- f(\overline{s}) \le d_{\Ss}(s,\overline{s})$ since $\|f\|_L \le 1$ and by the definition of $\piP[\pi'][p'] \mu_{\pi',p'}^{\gamma,\delta_{s}}$. The Lipschitz semi-norm $\lip{h_f}$ is computed in Lemma~\ref{lemma:lipSum}, resulting in $L_{p'}(1+L_{\pi'})$. Putting all together and exploiting the recursion, we obtain:
	\begin{align*}
	\wass{\mu_{\pi',p'}^{\gamma,\delta_{s}} }{\mu_{\pi',p'}^{\gamma,\delta_{\overline{s}}} } \le \frac{1-\gamma}{1-\gamma L_{p'}(1+L_{\pi'})} d_{\Ss}(s,\overline{s}) \quad \implies \quad g_f(s)  -  g_f(\overline{s}) \le  \frac{1}{1-\gamma  L_{p'}(1+L_{\pi'})} d_{\Ss}(s,\overline{s}),
	\end{align*}
	that concludes the proof.
\end{proof}

Thus, we have bounded the Wasserstein distance between the $\gamma$-discounted stationary distributions with the Wasserstein distance between the state transition kernels, averaged over the distribution $\piPD$. The resulting multiplicative constant involves the Lipschitz constants of the policy $\pi'$ and of the configuration $p'$. Therefore, remarkably, in order to obtain such a result it is not required that the policy-configuration pair $(\pi,p)$ is LC.
%
%\begin{remark}[Comparison with Total-Variation Bounds] It is worth noting that when setting the distance metrics over the state and action spaces as the \emph{discrete} metrics:
%\begin{align*}
%d_{\Ss}(s,\overline{s}) = \mathds{1}\{s\neq\overline{s}\}, \quad d_{\As}(a,\overline{a}) = \mathds{1}\{a\neq\overline{a}\}, \quad \forall (s,a),(\overline{s},\overline{a}) \in \SAs,
%\end{align*}
%the Wasserstein distance reduces to the \emph{total variation} (TV) divergence. It is easily shown that, in such a case, the Lipschitz constant of the state-state transition kernel reduces to $1$. Indeed, by recalling that the TV divergence is bounded by $1$, for every $s,\overline{s} \in \SAs$, we have:
%\begin{align}\label{eq:tv}
%	\text{TV}\left(\piP(\cdot|s), \piP(\cdot|s) \right) \le 1 \cdot \mathds{1}\{s\neq\overline{s}\} + 0 \cdot \mathds{1}\{s=\overline{s}\} = \mathds{1}\{s\neq\overline{s}\}.
%\end{align}
%Plugging Equation~\eqref{eq:tv} into Theorem~\ref{thr:distributionsBoundCoupled}, we obtain back the well-known TV bound between the $\gamma$-discounted stationary distributions~\cite{PirottaRPC13, SchulmanLAJM15, AchiamHTA17, MetelliPCR21}:
%\begin{align*}
%	\text{TV}\left( \piPD[\pi'][p'], \piPD\right) \le  \frac{\gamma}{1-\gamma }  \int_{\Ss}\piPD(\de s)  \text{TV}{\left( \piP[\pi'][p'](\cdot|s), \piP[\pi][p](\cdot|s) \right) }.
%\end{align*}
%\end{remark}

\subsection{Decoupled Bound}\label{sec:uncoupledGamma}

In some applications, it is useful to decouple the contribution of the configuration $p$ and that of the policy $\pi$ in the bound of the Wasserstein distance between the $\gamma$-discounted stationary distributions. Indeed, in several applicative scenarios of the Conf-MDP, as noted earlier, the configuration activity (\ie altering the transition model of the environment) and the policy learning are carried out by two different entities, possibly at different time scales. The following corollary provides a looser bound in which the contribution of policy and configurations are separated.

\begin{restatable}[]{coroll}{BoundDisj}\label{coroll:BoundDisj}
	Let $\ConfMDP$ be a Conf-MDP, let $(\pi,p),(\pi',p') \in \PP \times \PPP$ be two policy-configuration pairs such that $\pi'$ is $L_{\pi'}$-LC and $p' $ is $L_{p'}$-LC. Then, if $\gamma L_{p'}(1+L_{\pi'}) < 1$, it holds that:
    \begin{align*}
    \mathcal{W}\left( \piPD[\pi'][p'], \piPD\right) & \le  \frac{\gamma}{1-\gamma L_{p'}(1+L_{\pi'})}  \int_{\SAs}\piPD(\de s) \pi(\de a |s) \mathcal{W}{\left( p'(\cdot|s,a), p(\cdot|s,a) \right) } \\
    & \quad +\frac{\gamma L_{p'}}{1-\gamma L_{p'}(1+L_{\pi'})}  \int_{\Ss}\piPD(\de s)  \mathcal{W}{\left( \pi'(\cdot|s), \pi(\cdot|s) \right) }.
    \end{align*}
\end{restatable}

\subsection{Comparison with Existing Bounds}\label{sec:comp}
The bounds presented in the previous sections generalize several existing results present the literature. In particular, Corollary~\ref{coroll:BoundDisj} is an extension of Lemma 3 of~\cite{PirottaRB15} for the case of Conf-MDPs in which we allow for the modification of the configuration (\ie transition model) too. Furthermore, even in the non-configurable setting, \ie when setting $p=p'$ a.s., our result is stronger as it involves the  Wasserstein distance between policies \emph{averaged} over the $\gamma$-discounted stationary distribution instead of its \emph{supremum} over the state space:
\begin{align*}
	\underbrace{ \int_{\Ss}\piPD(\de s)  \mathcal{W}{\left( \pi'(\cdot|s), \pi(\cdot|s) \right) }}_{\text{Our Corollay~\ref{coroll:BoundDisj}}} \le \underbrace{\sup_{s \in \Ss}  \mathcal{W}{\left( \pi'(\cdot|s), \pi(\cdot|s) \right)}}_{\text{Lemma 3 of~\cite{PirottaRB15}}}.
\end{align*} 

A result connected to ours, with a bound involving the modification of the transition model, appeared in~\cite{AsadiML18} for model-based RL. However, the provided equation holds under a uniform bound on the Wasserstein distance between transition models $\sup_{(s,a) \in \SAs} \wass{p'(\cdot|s,a)}{p(\cdot|s,a)}$, that, instead, we relax considering the Wasserstein distance averaged over  the $\gamma$-discounted stationary distribution, \ie $ \int_{\SAs} \mu_{\pi,p}(\de s) \pi(\de a|s) \wass{p'(\cdot|s,a)}{p(\cdot|s,a)}$. 

Finally, in~\cite{truly2022} a bound on the Wasserstein distance between the $\gamma$-discounted stationary distributions, as a function of the modification of the policy only, is proposed under a different set of regularity assumptions. In particular, two conditions on the transition kernel are enforced for every $\pi,\pi' \in \PP$ and $\nu,\nu' \in \PM{\Ss}$:
\begin{align*}
	& \wass{\int_{\Ss} \nu(\de s) p_{\pi}(\cdot|s)}{ \int_{\Ss} \nu(\de s) p_{\pi'}(\cdot|s)} \le \widetilde{L}_{\pi} \sup_{s \in \Ss} \wass{\pi(\cdot|s)}{\pi'(\cdot|s)}, \\
	& \wass{\int_{\Ss} \nu(\de s) p_{\pi}(\cdot|s)}{ \int_{\Ss} \nu'(\de s) p_{\pi}(\cdot|s)} \le \widetilde{L}_{\nu} \wass{\nu}{\nu'},
\end{align*}
where $\widetilde{L}_{\pi}$ and $ \widetilde{L}_{\nu}$ are suitable Lipschitz constants. These conditions, however, are not comparable with the ones considered in the present paper (Section~\ref{sec:smooth}), as they evaluate how fast the transition kernel changes when altering either the policy or the initial state distribution. Consequently, it is hard to argue whether of the resulting bound is tighter than ours. Nevertheless, it is worth noting that ours is the only one that involves the average Wasserstein distance between policies, rather than the supremum over the state (or state-action) space.

\section{Relative Advantage Functions}\label{sec:FINITEadv}
In this section, we revise the \emph{relative advantage functions} for Conf-MDPs that are needed for the derivation of the performance improvement bounds, as introduced in~\cite{metelli2018configurable, phdthesis}. The notion of \emph{advantage function} exists in the traditional MDP setting and evaluates the one-step improvement in executing an action $a \in \As$ compared to executing the current policy $\pi \in \PP$~\cite{Puterman94}. 

We now extend the advantage function notion to the Conf-MDP setting, properly accounting for the presence of the agent and the configuration. Specifically, we introduce three notions: the \emph{policy} advantage function, the \emph{configuration} advantage function, and the \emph{coupled} advantage function, respectively defined for every $(s,a,s') \in \SASs$ as:
\begin{align*}
	& A_{\pi,p}(s,a) = \QpiP(s,a) - \VpiP(s), \\
	& A_{\pi,p}(s,a,s') = \UpiP(s,a,s') - \QpiP(s,a), \\
	& \widetilde{A}_{\pi,p}(s,a,s') = \UpiP(s,a,s') - \VpiP(s) = A_{\pi,p}(s,a,s') + A_{\pi,p}(s,a).
\end{align*}
These functions evaluate the one-step performance improvement obtained in state $s \in \Ss$ by either playing action $a \in \As$, for the policy advantage, selecting the next state $s' \in \Ss$ given that action $a \in \As$ was played, for the model advantage, or both for the coupled advantage, compared to playing policy $\pi$ and transition model $p$. 

To quantify the one-step improvement in performance attained by a new policy $\pi'$ or transition model $p'$ when the current policy is $\pi$ and the current model is $p$, we introduce the \emph{(decoupled) relative advantage functions}~\citep{kakade2002approximately} defined for every state-action pair $\sasa$ as:
\begin{align*}
	&A_{\pi,p}^{\pi',p}(s) = \int_{\mathcal{A}} \pi'(\de a|s) A_{\pi,p}(s,a),\\
    &A_{\pi,p}^{\pi,p'}(s,a) = \int_{\mathcal{S}} p'(\de s'|s,a) A_{\pi,p}(s,a,s'),
\end{align*}
and the corresponding expected values under the $\gamma$-discounted distributions:
\begin{align*}
& \mathbb{A}_{\pi,p,\isd}^{\pi',p} = \int_{\mathcal{S}} \mu_{\pi,p} (\de s) A_{\pi,p}^{\pi',p}(s) \\
& \mathbb{A}_{\pi,p,\isd}^{\pi,p'} = \int_{\mathcal{S}} \int_{\mathcal{A}} \mu_{\pi,p}(\de s) \pi(\de a|s) A_{\pi,p}^{\pi,p'}(s,a).
\end{align*}
In order to account for the combined effect of choosing the action with a new policy $\pi'$ and the next state with the new configuration $p'$, we introduce the  \emph{coupled relative advantage function} defined for every state $s \in \Ss$ as:
\begin{equation*}
		A_{\pi,p}^{\pi',p'}(s) = \int_{\mathcal{S}} \int_{\mathcal{A}}  \pi'(\de a|s)p'(\de s'|s,a) \widetilde{A}_{\pi,p}(s,a,s').
\end{equation*}
Thus, $A_{\pi,p}^{\pi',p'}$ evaluates the one-step improvement obtained by the new policy-configuration pair $(\pi',p')\in \PP \times \PPP$ over
the current one $(\pi,p) \in \PP \times \PPP$, \ie the local gain in performance achieved by selecting an action with $\pi'$ and the next state with $p'$. The corresponding expectation under the $\gamma$-discounted distribution is given by: 
\begin{equation*}
\mathbb{A}_{\pi,p,\isd}^{\pi',p'} = \int_{\mathcal{S}} \mu_{\pi,p} (\de s) A_{\pi,p}^{\pi',p'}(s).
\end{equation*}
To lighten the notation, we remove the subscript of the initial state distribution $\isd$ whenever clear from the context. Thus, we simply write $\mathbb{A}_{\pi,p}^{\pi',p}$, $\mathbb{A}_{\pi,p}^{\pi,p'}$, and $\mathbb{A}_{\pi,p}^{\pi',p'}$. The following result relates the coupled relative advantage function with the corresponding (decoupled) relative advantage functions.

\begin{restatable}[Lemma A.1 of \cite{metelli2018configurable}]{lemma}{lemmaAdvantage}\label{thr:lemmaAdvantage}
	Let $A_{\pi,p}^{\pi',p'}$ be the coupled relative advantage function, $A_{\pi,p}^{\pi',p}$ and $ A_{\pi,p}^{\pi,p'}$ be the (decoupled) policy and configuration relative advantage functions respectively. Then, for every state $s \in \Ss$ it holds that:
	\begin{equation*}
		A_{\pi,p}^{\pi',p'}(s) = A_{\pi,p}^{\pi',p}(s) + \int_{\mathcal{A}} \pi'(\de a|s) A_{\pi,p}^{\pi,p'}(s,a).
	\end{equation*}
\end{restatable} 

This result has a meaningful interpretation. In order to assess the one-step performance improvement in state $s \in \Ss$ obtained moving from policy-configuration pair $(\pi,p) $ to the new one $(\pi',p')$, \ie $A_{\pi,p}^{\pi',p'}(s)$, we can add the contribution of two terms: (i) the one-step performance improvement due to the policy, \ie $A_{\pi,p}^{\pi',p}(s)$; (ii) the one-step performance improvement due to the configuration, \ie $A_{\pi,p}^{\pi,p'}(s,a)$, with the action sampled from the new policy $\pi'$.

\section{Bound on the Performance Improvement}\label{sec:boundPI}
In this section, we make use of the results provided in the previous sections, to derive lower bounds on the performance improvement yielded by changing the policy and the configuration. We start the section by revising the performance difference lemma for Conf-MDPs, as derived in~\cite{metelli2018configurable}. Then, we proceed at proving the performance improvement lower bounds. Specifically, in Section~\ref{sec:copPP}, we derive a tight lower bound that, unfortunately, is hardly usable in practice since it involves the presence of non-easily computable quantities and keeps together the contributions of the policy and the configuration (coupled). Then, in Section~\ref{sec:Dec1L}, we derive a looser bound that separates the effect of policy and configuration (decoupled) and just requires the LC property.

\begin{restatable}[Performance Difference Lemma - Theorem 3.1 of~\cite{metelli2018configurable}]{thr}{perfImprovement}
\label{thr:perfImprovement}
Let $\ConfMDP$ be a Conf-MDP. The performance improvement of policy-configuration pair $(\pi', p') \in \PP \times \PPP$ over $(\pi, p) \in \PP \times \PPP$ is given by:
    \begin{equation*}
    	J_{\pi',p'} - J_{\pi,p} = \frac{1}{1-\gamma} \int_{\mathcal{S}} \piPD[\pi'][p'](\de s) A_{\pi,p}^{\pi',p'}(s) .
    \end{equation*}
\end{restatable} 

This result represents the extension of the well-known performance difference lemma of~\cite{kakade2002approximately} to the case in which we are allowed to change the transition model (\ie configuration) too. The meaning of the result can be summarized as follows. The performance improvement $J_{\pi',p'} - J_{\pi,p}$ attained by moving from pair $(\pi,p)$ to pair $(\pi',p')$ is computed by means of the coupled relative advantage function $A_{\pi,p}^{\pi',p'}$, averaged over the $\gamma$-discounted stationary distribution $\piPD[\pi'][p']$, induced by the new pair $(\pi',p')$.

\subsection{Coupled Bound}\label{sec:copPP}
From a practical perspective, the expression of Theorem~\ref{thr:perfImprovement} is an equality but cannot be directly used since it requires computing (or estimating) an expectation \wrt the $\gamma$-discounted distribution $\piPD[\pi'][p']$, depending on the new pair $(\pi',p')$ that we have not access to. Instead, we would like to obtain a performance improvement lower bound that can be estimated by using the current pair $(\pi,p)$. The following result serves the purpose.

\begin{restatable}[Coupled Bound]{thr}{boundCoupled}\label{thr:boundCoupled}
	Let $\ConfMDP$ be a $L_r$-LC Conf-MDP, let $\pi,\pi' \in \PP$ be two $L_{\pi},L_{\pi'}$-LC policies respectively, and let $p,p'\in\PPP$ be two $L_{p},L_{p'}$-LC configurations. The performance improvement of policy-configuration pair $(\pi', P') $ over $(\pi, P) $ is lower bounded as:
    \begin{equation*}
		\underbrace{J_{\pi', {p}'} - J_{\pi, {p}}}_{\substack{\text{performance}\\\text{improvement}}} \ge \underbrace{\frac{1}{1-\gamma} \mathbb{A}_{\pi,{p}}^{\pi',{p}'} }_{\text{advantage}} - \underbrace{\frac{\gamma}{(1-\gamma)(1-\gamma L_{p'}(1+L_{\pi'}))} \, \left\| A_{\pi,p}^{\pi',p'} \right\|_L \int_{\Ss} \piPD(\de s) \mathcal{W}{\left(\piP[\pi'][p'](\cdot|s) , \piP(\cdot|s) \right)}  }_{\text{dissimilarity penalization}}.
	\end{equation*}
	%where $\spans{A^{{P'},\pi'}_{\pi,P}} = \sup_{s \in \Ss} \left\{A^{{P'},\pi'}_{\pi,P}(s) \right\} - \inf_{s \in \Ss} \left\{A^{{P'},\pi'}_{\pi,P}(s) \right\} $.
\end{restatable}

\begin{proof}	
	Exploiting the bounds on the $\gamma$-discounted state distributions difference (Theorem~\ref{thr:distributionsBoundCoupled}) we can easily derive the performance improvement bound:
	\begin{align}
		J_{\pi',p'} - J_{\pi,p} & = \frac{1}{1-\gamma} \int_{\mathcal{S}}\piPD[\pi'][p'](\de s) A_{\pi,p}^{\pi',p'}(s)  \notag\\
		& = \frac{1}{1-\gamma} \int_{\mathcal{S}} \piPD(\de s) A_{\pi,p}^{\pi',p'}(s) +  \frac{1}{1-\gamma} \int_{\mathcal{S}} \left( \piPD[\pi'][p'](\de s) - \piPD(\de s) \right) A_{\pi,p}^{\pi',p'}(s)   \label{line:cb1}\\
		& \ge \frac{\mathbb{A}_{\pi,p}^{\pi',p'}}{1-\gamma} - \frac{1}{1-\gamma} \bigg|  \int_{\mathcal{S}} \left(\piPD[\pi'][p'](\de s) - \piPD(\de s)  \right) A_{\pi,p}^{\pi',p'}(s) \bigg| \label{line:cb2}\\
		& \geq   \frac{\mathbb{A}_{{\pi,p}}^{\pi',p'}}{1 - \gamma} -  \frac{\gamma}{(1-\gamma)(1-\gamma L_{p'}(1+L_{\pi'}))}  \int_{\Ss} \piPD(\de s) {\mathcal{W}}\left( \piP[\pi'][p'](\cdot|s) , \piP(\cdot|s) \right)  \left\| A_{\pi,p}^{\pi',p'}\right\|_L, \label{line:cb4}
	\end{align}	
	where line~\eqref{line:cb2} follows from line~\eqref{line:cb1} by observing that $b \ge -|b|$ for any $b \in \Reals$ and line~\eqref{line:cb4} is obtained by using Theorem~\ref{thr:distributionsBoundCoupled} and the definition of 
Lipschitz semi-norm.
\end{proof}

The resulting bound is made of two terms, as the commonly known bounds in the literature~\citep{kakade2002approximately,PirottaRPC13, MetelliPCR21}. The first term, \emph{advantage}, accounts for the improvement in performance that can be obtained locally by replacing the pair $(\pi,p)$ with the new one $(\pi',p')$. Instead, the second term, \emph{dissimilarity penalization}, disincentives a choice of the next pair that is too dissimilar from the current one. As expected, the multiplicative constant in front of the Wasserstein distance between the state-state transition kernels involves the Lipschitz constants of all the constitutive elements of the Conf-MDP, policy, and configuration.

\subsection{Decoupled Bound}\label{sec:Dec1L}
Despite its generality, the bound of Theorem~\ref{thr:boundCoupled} is hardly usable in practice because of the need for computing the Lipschitz semi-norm $\| A_{\pi,p}^{\pi',p'}\|_L$. Intuitively, this factor strictly depends on the similarity between the policy-configuration pairs $(\pi,p)$ and $(\pi',p')$. Indeed, when $(\pi,p) = (\pi',p')$ a.s., we have that $A_{\pi,p}^{\pi',p'}(s) = 0$ for all $s \in \Ss$ and, consequently, $\| A_{\pi,p}^{\pi',p'}\|_L = 0$. Unfortunately, obtaining a bound on this Lipschitz semi-norm that depends on some form of distance between the pairs  $(\pi,p)$ and $(\pi',p')$ is technically challenging and may require additional assumptions. Thus, we defer this goal to future works, in this section, we focus on a simpler (and looser) bound that makes use of the LC condition only, accepting to bound $\| A_{\pi,p}^{\pi',p'}\|_L$ with a constant independent from the distance between the pairs  $(\pi,p)$ and $(\pi',p')$. The following result provides such a bound.

\begin{restatable}[Decoupled Bound for Lipschitz Conf-MDPs]{thr}{boundLoose}\label{thr:boundLoose}
Let $\ConfMDP$ be a $L_r$-LC Conf-MDP, let $\pi,\pi' \in \PP$ be two $L_{\pi},L_{\pi'}$-LC policies respectively, and let $p,p'\in\PPP$ be two $L_{p},L_{p'}$-LC configurations. The performance improvement of policy-configuration pair $(\pi', P') $ over $(\pi, P) $ is lower bounded as:
    \begin{align*}
		J_{\pi', {p}'} - J_{\pi, {p}} \ge \frac{\mathbb{A}_{\pi,{p}}^{\pi',{p}} + \mathbb{A}_{\pi,{p}}^{\pi,{p}'}}{1-\gamma}  - &  \Bigg(  c_1 \int_{\SAs}\piPD(\de s) \pi(\de a |s) \mathcal{W}{\left( p'(\cdot|s,a), p(\cdot|s,a) \right) } \\
    & \qquad +c_2  \int_{\Ss}\piPD(\de s)  \mathcal{W}{\left( \pi'(\cdot|s), \pi(\cdot|s) \right) } \Bigg)  ,
	\end{align*}
	where $c_1$ and $c_2$ are constant whose values are made explicit in the proof and depend on $\gamma$, $L_r$, $L_{\pi}$, $L_{\pi'}$, $L_p$, and $L_{p'}$.
\end{restatable}

Some observations are in order. First, in order to obtain this result, we need to require that the Conf-MDP and both pairs $(\pi,p)$ and $(\pi',p')$ to be LC. Indeed, the constants $c_1$ and $c_2$ will depend on the Lipschitz semi-norm of the Conf-MDP elements, and on the Lipschitz constants of both policies and both configurations. Second, Theorem~\ref{thr:boundLoose} represents an decoupled bound, that separates the effects of the policy and the configuration. This is made evident in the advantage term too that is now replaced with the sum of the expected decoupled relative advantages. Third, and more important, the bound displays a linear dependence on the Wasserstein distance between the transition models and the policies. Compared to well-known bounds in the literature in which the square of the divergence is present (typically TV divergence), such as~\cite{PirottaRPC13}, this represents a looser dependence.

\section{Conclusions}
In this paper, we have investigated the performance improvement bounds for Conf-MDP under Lipschitz regularity. We have first derived bounds on the Wasserstein distance between $\gamma$-discounted stationary distributions that generalize existing bounds in the literature. Then, we have provided two performance improvement lower bounds. Future promising works include the employment of the derived bounds for devising safe and trust region learning algorithms.

\bibliographystyle{apalike}
\bibliography{ms}

\section{Proofs Omitted in the Main Paper}
In this appendix, we report the proofs we have omitted in the main paper.

\BoundDisj*
\begin{proof}
	Let $s \in \Ss$, we start from Theorem~\ref{thr:distributionsBoundCoupled} and apply triangular inequality:
	\begin{align*}
		\mathcal{W}{\left( \piP[\pi'][p'](\cdot|s), \piP[\pi][p](\cdot|s) \right) } \le \mathcal{W}{\left( \piP[\pi'][p'](\cdot|s), \piP[\pi][p'](\cdot|s) \right) } + \mathcal{W}{\left( \piP[\pi][p'](\cdot|s), \piP[\pi][p](\cdot|s) \right) }. 
	\end{align*}
	For the first term, we have:
	\begin{align}
		\int_{\Ss} \piPD(\de s) \mathcal{W}{\left( \piP[\pi'][p'](\cdot|s), \piP[\pi][p'](\cdot|s) \right) } & = \int_{\Ss} \piPD(\de s) \suplip \left| \int_{\Ss} \left(  \piP[\pi'][p'](\de s'|s)- \piP[\pi][p'](\de s'|s) \right)  f(s') \right| \notag\\
		& = \int_{\Ss} \piPD(\de s) \suplip \left|  \int_{\As} \left( \pi'(\de a|s) - \pi(\de a|s) \right) \underbrace{\int_{\Ss}  p'(\de s'|s,a)  f(s')}_{\eqqcolon h_f(s,a)} \right| \notag\\
		& \le \suplip \lip{h_f} \int_{\Ss} \piPD(\de s) \wass{ \pi'(\cdot|s)}{ \pi(\cdot|s) },\label{eq:901}
	\end{align}
	where the lats passage follows from the definition of Wasserstein distance and Lipschitz semi-norm. We now compute the Lipschitz semi-norm $\lip{h_f}$. Let $(s,a),(\overline{s},\overline{a}) \in \SAs$:
	\begin{align*}
		h_f(s,a) - h_f(\overline{s},\overline{a})  = \int_{\Ss} \left( p'(\de s'|s,a) - p'(\de s'|\overline{s},\overline{a}) \right) f(s') \le \wass{p'(\cdot|s,a)}{p'(\cdot|\overline{s},\overline{a})} \le L_{p'} d_{\SAs}((s,a),(\overline{s},\overline{a})).
	\end{align*}
	For the second term, we have:
	\begin{align*}
		\int_{\Ss} \piPD(\de s) \mathcal{W}{\left( \piP[\pi][p'](\cdot|s), \piP[\pi][p](\cdot|s) \right) } & = \int_{\Ss} \piPD(\de s)  \suplip \left| \int_{\Ss} \left(  \piP[\pi][p'](\de s'|s)- \piP[\pi][p](\de s'|s) \right)  f(s') \right| \notag\\
		& = \int_{\Ss} \piPD(\de s) \suplip \left|  \int_{\As} \pi(\de a|s) \int_{\Ss} \left(  p'(\de s'|s,a)- p(\de s'|s,a) \right)  f(s') \right| \notag\\
		& \le \int_{\Ss} \piPD(\de s) \int_{\As} \pi(\de a|s)  \suplip \left| \int_{\Ss} \left(  p'(\de s'|s,a)- p(\de s'|s,a) \right)  f(s') \right|\notag \\
		& = \int_{\Ss} \piPD(\de s) \int_{\As} \pi(\de a|s)  \wass{ p'(\de s'|s,a)}{p(\de s'|s,a)},\label{eq:902}
	\end{align*}
	where the inequality follows from Jensen's inequality. Putting all together, we get the result.
\end{proof}

\lemmaAdvantage*
\begin{proof}
Let $s\in \Ss$, let us consider the following derivation:
	\begin{align}
		A_{\pi,p}^{\pi',p'}(s) & = \int_{\mathcal{A}} \int_{\mathcal{S}} \pi'(\de a|s)  p'(\de s'|s,a) \UpiP(s,a,s')  - \VpiP(s)  \notag \\
			& = \int_{\mathcal{A}} \int_{\mathcal{S}} \pi'(\de a|s)  p'(\de s'|s,a) \UpiP(s,a,s')  - \VpiP(s) \pm \int_{\mathcal{A}} \int_{\mathcal{S}} \pi'(\de a|s)  p(\de s'|s,a) \UpiP(s,a,s')  \notag \\
			& = \int_{\mathcal{A}} \int_{\mathcal{S}} \pi'(\de a|s)  p(\de s'|s,a) \UpiP(s,a,s') - \VpiP(s)\\
			& \quad +  \int_{\mathcal{A}} \int_{\mathcal{S}} \pi'(\de a|s)  \left( p'(\de s'|s,a) - p(\de s'|s,a) \right) \UpiP(s,a,s')  \notag \\
			& = \int_{\mathcal{A}} \pi'(\de a|s)  \QpiP(s,a) \mathrm{d}a - \VpiP(s)  + \int_{\mathcal{A}}  \pi'(\de a|s) \int_{\mathcal{S}} \left( p'(\de s'|s,a) - p(\de s'|s,a) \right) \UpiP(s,a,s')  \label{linep:0}\\
			& = A_{\pi,p}^{\pi',p}(s) + \int_{\mathcal{A}} \pi'( \mathrm{d}a|s) A_{\pi,p}^{\pi,p'}(s,a), \label{linep:2}
	\end{align}
	where line~\eqref{linep:0} is obtained by recalling that $\QpiP(s,a) = \int_{\Ss} p(\de s'|s,a) \UpiP(s,a,s')$, the first addendum of line~\eqref{linep:2} follows from observing that: 
	$$A_{\pi,p}^{\pi',p}(s) = \int_{\mathcal{A}} \pi'(\de a|s)A^{\pi,p}(s,a) = \int_{\mathcal{A}} \pi'(\de a|s)\left( \QpiP(s,a) - \VpiP(s) \right),$$
and similarly the second addendum of line~\eqref{linep:2} comes from the identity:
\begin{align*}
A_{\pi,p}^{\pi,p'}(s,a) & = \int_{\mathcal{S}} p'(\de s'|s,a) A^{\pi,p}(s,a,s') = \int_{\mathcal{S}} p'(\de s'|s,a) \left( \UpiP(s,a,s') - \QpiP(s,a) \right).
\end{align*}
\end{proof}

\perfImprovement*
\begin{proof}
	Let us start from the definition of $J^{\pi',P'}$:
	\begin{align}
		(1-\gamma)  J_{\pi', p'} & = \int_{\mathcal{S}} \int_{\mathcal{A}} \int_{\Ss} \piPD[\pi'][p'](\de s) \pi'(\de a|s) p'(\de s'|s,a) r(s,a,s') \nonumber \\
        & = \int_{\mathcal{S}} \int_{\mathcal{A}} \int_{\Ss} \piPD[\pi'][p'](\de s) \pi'(\de a|s) p'(\de s'|s,a) r(s,a,s')  \pm \int_{\mathcal{S}} \piPD[\pi'][p'](\de s) \VpiP(s)  \label{line:impr1}\\ 
        & = \int_{\mathcal{S}} \int_{\mathcal{A}} \int_{\Ss} \piPD[\pi'][p'](\de s) \pi'(\de a|s) p'(\de s'|s,a) r(s,a,s') \label{line:impr2} \\ 
        & \quad +  \int_{\mathcal{S}} \left( (1-\gamma)\isd(\de s') + \gamma \int_{\mathcal{S}} \int_{\mathcal{A}} \piPD[\pi'][p'](\de s) \pi'(\de a|s)p'(\de s'|s,a)  \right) \VpiP(s')  \nonumber \\
        & \quad - \int_{\mathcal{S}}  \piPD[\pi'][p'](\de s) \VpiP(s)   \nonumber \\
        & = \int_{\mathcal{S}}  \piPD[\pi'][p'](\de s) \left( \int_{\mathcal{A}} \pi'(\de a|s)   \int_{\mathcal{S}} p'(\de s'|s,a) \left( r(s,a,s') + \gamma \VpiP(s') \right) - \VpiP(s) \right)  \notag\\
        & \quad + \int_{\mathcal{S}} \isd(\de s') \VpiP(s')  \nonumber \\
        & =  \int_{\mathcal{S}} \piPD[\pi'][p'](\de s) A_{\pi,p}^{\pi',p'}(s)  + (1-\gamma) J_{\pi,p},\label{line:impr4}
    \end{align}
	where we have made use of the recursive formulation of $\piPD[\pi'][p']$ (Equation~\ref{eq:discDistr}) to rewrite line (\ref{line:impr1}) into line (\ref{line:impr2}) and line~\eqref{line:impr4} follows by observing that $\int_{\mathcal{S}} \isd(\de s') \VpiP(s')  = J_{\pi,p}$ and using the definition $\UpiP(s,a,s') =  r(s,a,s') + \gamma \VpiP(s')$.
\end{proof}

\boundLoose*
\begin{proof}
	We provide a bound on the Lipschitz semi-norm $\left\| A_{\pi,p}^{\pi',p'} \right\|_L $. Let $s,\overline{s} \in \Ss$:
	\begin{align*}
	 A_{\pi,p}^{\pi',p'}(s) - A_{\pi,p}^{\pi',p'}(\overline{s}) & = \int_{\SAs} \pi'(\de a|s) p'(\de s'|s,a) \UpiP(s,a,s') - \VpiP(s) \\
	 & \quad -  \int_{\SAs} \pi'(\de a |\overline{s}) p'(\de s'|\overline{s},a) \UpiP(\overline{s},a,s') - \VpiP(\overline{s}) \\
	 & = \underbrace{\int_{\SAs} \pi'(\de a|s) p'(\de s'|s,a)( \UpiP(s,a,s') - \UpiP(\overline{s},a,s'))}_{\text{(a)}} \\
	 & \quad + \underbrace{\int_{\SAs} \left(\pi'(\de a|s) p'(\de s'|s,a) - \pi'(\de a |\overline{s}) p'(\de s'|\overline{s},a) \right)\UpiP(\overline{s},a,s') }_{\text{(b)}} \\
	 & \quad + \underbrace{\VpiP(\overline{s}) - \VpiP(s)}_{\text{(c)}}\\
	 %& \le \sup_{(a,s') \in \mathcal{A}\times \mathcal{S}} \lip{\UpiP(\cdot,a,s')} d_{\Ss}(s,\overline{s}) + \lip{\VpiP}d_{\Ss}(s,\overline{s})  \\
	 %& \quad + \int_{\As} \left(\pi'(\de a|s)  - \pi'(\de a |\overline{s})  \right)\underbrace{\int_{\Ss} p'(\de s'|\overline{s},a) \UpiP(\overline{s},a,s') }_{\eqqcolon h(\overline{s},a)}\\
	 %& \quad + \int_{\SAs} \pi'(\de a|s) \left(  p'(\de s'|s,a) - p'(\de s'|\overline{s},a)  \right) \UpiP(\overline{s},a,s').
	\end{align*}
	\paragraph{Term (a)} Concerning term (a), we exploit the Lipschitz semi-norm of the $\UpiP$ function, to obtain:
	\begin{align*}
		\int_{\SAs} \pi'(\de a|s) p'(\de s'|s,a)( \UpiP(s,a,s') - \UpiP(\overline{s},a,s')) \le \sup_{(a,s') \in \mathcal{A}\times \mathcal{S}} \lip{\UpiP(\cdot,a,s')} d_{\Ss}(s,\overline{s}).
	\end{align*}		
	\paragraph{Term (b)}  Concerning term (b), we need to perform further manipulations:
	\begin{align*}
	\int_{\SAs} &  \left(\pi'(\de a|s) p'(\de s'|s,a) - \pi'(\de a |\overline{s}) p'(\de s'|\overline{s},a) \right)  \UpiP(\overline{s},a,s') = \\
	&\qquad = \int_{\As} \left(\pi'(\de a|s)  - \pi'(\de a |\overline{s})  \right)\underbrace{\int_{\Ss} p'(\de s'|\overline{s},a) \UpiP(\overline{s},a,s') }_{\eqqcolon h(\overline{s},a)} \\
	& \qquad\quad + \int_{\SAs} \pi'(\de a|s) \left(  p'(\de s'|s,a) - p'(\de s'|\overline{s},a)  \right) \UpiP(\overline{s},a,s')
\end{align*}		
	For the first integral, we need to compute the Lipschitz semi-norm of $h(\overline{s},a)$. Let $(\overline{s},a),(\overline{s}',a') \in \SAs$:
	\begin{align*}
	h(\overline{s},a) - h(\overline{s}',a') & = \int_{\Ss} p'(\de s'|\overline{s},a) \UpiP(\overline{s},a,s')  - \int_{\Ss} p'(\de s'|\overline{s}',a') \UpiP(\overline{s}',a',s')  \\
	& = \int_{\Ss} \left(p'(\de s'|\overline{s},a) - p'(\de s'|\overline{s}',a') \right)\UpiP(\overline{s},a,s') \\
	& \quad + \int_{\Ss} p'(\de s'|\overline{s}',a') (\UpiP(\overline{s},a,s') -\UpiP(\overline{s}',a',s'))\\
	& \le \sup_{s,a \in \SAs} \lip{\UpiP(s,a,\cdot)} \wass{p'(\cdot|\overline{s},a)}{p'(\cdot|\overline{s}',a')} + \sup_{s' \in \Ss} \lip{\UpiP(\cdot,\cdot,s')} d_{\SAs}((\overline{s},a),(\overline{s}',a')) \\
	& \le \left( \sup_{s,a \in \SAs} \lip{\UpiP(s,a,\cdot)} L_{p'} +  \sup_{s' \in \Ss} \lip{\UpiP(\cdot,\cdot,s')} \right)d_{\SAs}((\overline{s},a),(\overline{s}',a')).
	\end{align*}
	Thus, we have:
	\begin{align*}
		\int_{\As} \left(\pi'(\de a|s)  - \pi'(\de a |\overline{s})  \right) h(\overline{s},a) & \le  \wass{\pi'(\cdot|s)}{\pi'(\cdot |\overline{s}) } \left( \sup_{s,a \in \SAs} \lip{\UpiP(s,a,\cdot)} L_{p'} +  \sup_{s' \in \Ss} \lip{\UpiP(\cdot,\cdot,s')} \right) \\
		& \le L_{\pi'} \left( \sup_{s,a \in \SAs} \lip{\UpiP(s,a,\cdot)} L_{p'} +  \sup_{s' \in \Ss} \lip{\UpiP(\cdot,\cdot,s')} \right) d_{\Ss}(s,\overline{s}).
	\end{align*}
	For the second integral, we can simply proceed as follows:
	\begin{align*}
		\int_{\SAs} \pi'(\de a|s)&  \left(  p'(\de s'|s,a) - p'(\de s'|\overline{s},a)  \right) \UpiP(\overline{s},a,s') \\
		& \le \sup_{(s,a) \in \SAs} \lip{\UpiP(s,a,\cdot)} \int_{\SAs} \pi'(\de a|s) \wass{p'(\de s'|s,a)}{p'(\de s'|\overline{s},a) } \\
		& \le \sup_{(s,a) \in \SAs} \lip{\UpiP(s,a,\cdot)} L_{p'} d_{\Ss}(s,\overline{s}).
	\end{align*}
	\paragraph{Term (c)} Concerning term (c), it is immediate to verify that $\VpiP(\overline{s}) - \VpiP(s) \le \lip{\VpiP} d_{\Ss}(s,\overline{s})$.
	
		Finally, we observe that $\lip{\UpiP} \ge \max\left\{ \sup_{s,a \in \SAs} \lip{\UpiP(s,a,\cdot)}, \sup_{s' \in \Ss} \lip{\UpiP(\cdot,\cdot,s')}\right\}$. Thus, $\lip{h} \le (1+L_{p'}) \lip{\UpiP}$. With simple algebraic manipulation, we derive that:
		\begin{align*}
		\left\| A_{\pi,p}^{\pi',p'} \right\|_L & \le \lip{\UpiP} + \lip{\UpiP} L_{\pi'} (1+L_{p'}) + L_{p'}\lip{\UpiP} + \lip{\VpiP}\\
		& \le \lip{\UpiP} (1 + (L_{\pi'}+1)(L_{p'}+1)) \le 2\lip{\UpiP} (L_{\pi'}+1)(L_{p'}+1).
		\end{align*}
		By using Lemma~\ref{lemma:decompAdv}, we have:
		\begin{align*}
		 \mathbb{A}_{\pi,p}^{\pi',p'} & \ge  \mathbb{A}_{\pi,p}^{\pi,p'} + \mathbb{A}_{\pi,p}^{\pi',p} - \int_{\Ss} \mu_{\pi,p}(\de s) \wass{\pi'(\cdot|s)}{\pi(\cdot|s)} \lip{A_{\pi,p}^{\pi,p'}} \\
		 & \ge  \mathbb{A}_{\pi,p}^{\pi,p'} + \mathbb{A}_{\pi,p}^{\pi',p} - \int_{\Ss} \mu_{\pi,p}(\de s) \wass{\pi'(\cdot|s)}{\pi(\cdot|s)} \left\| A_{\pi,p}^{\pi',p'} \right\|_L \\
		 & \ge \mathbb{A}_{\pi,p}^{\pi,p'} + \mathbb{A}_{\pi,p}^{\pi',p} -2 \lip{\UpiP} (L_{\pi'}+1)(L_{p'}+1) \int_{\Ss} \mu_{\pi,p}(\de s) \wass{\pi'(\cdot|s)}{\pi(\cdot|s)},
		\end{align*}
		having observed that $\lip{A_{\pi,p}^{\pi,p'}} \le \left\| A_{\pi,p}^{\pi',p'} \right\|_L$. By exploiting Corollary~\ref{coroll:BoundDisj}, we get the result, having defined the constants $c_1$ and $c_2$ as follows:
		\begin{align*}
		c_1 = \lip{\UpiP}  \frac{2 \gamma (L_{\pi'}+1)(L_{p'}+1)}{(1-\gamma)(1-\gamma L_{p'}(1+L_{\pi'}))}, \qquad c_2 = \lip{\UpiP}  \frac{2 (1+\gamma L_{p'}) (L_{\pi'}+1)(L_{p'}+1)}{(1-\gamma)(1-\gamma L_{p'}(1+L_{\pi'}))} .
		\end{align*}
\end{proof}

\section{Technical Lemmas}\label{apx:tech}

In this appendix, we provide additional technical results that are employed in the proofs of the results presented in the main paper.

\begin{lemma}\label{lemma:lipSum}
Let $h_f(s) = \int_{\Ss} \piP[\pi'][p'](\de s'|s)   f(s')$ with $\|f\|_L \le 1$. Then, it holds that $\lip{h_f} \le L_{p'}(1+L_{\pi'})$.
\end{lemma}

\begin{proof}
For $s,\overline{s} \in \Ss$, we have:
	\begin{align*}
		h_f(s) - h_f(\overline{s})  & = \int_{\Ss} \left( \piP[\pi'][p'](\de s'|s)  -  \piP[\pi'][p'](\de s'|\overline{s}) \right) f(s') \\
		& = \int_{\Ss} \int_{\As} \left( \pi'(\de a|s) p'(\de s'|s,a)  -  \pi'(\de a|\overline{s}) p'(\de s'|\overline{s},a) \right) f(s') \\
		& =  \int_{\As}  \pi'(\de a|s) \int_{\Ss} \left( p'(\de s'|s,a)  -  p(\de s'|\overline{s},a) \right) f(s') + \int_{\As}  \left( \pi'(\de a|s) -\pi'(\de a|\overline{s}) \right) \underbrace{\int_{\Ss} p'(\de s'|s,a) f(s')}_{\eqqcolon l_f(s,a)}.
	\end{align*}
	For the first integral, we simply obtain:
	\begin{align*}
	\int_{\As}  \pi'(\de a|s) \int_{\Ss} \left( p'(\de s'|s,a)  -  p(\de s'|\overline{s},a) \right) f(s') \le \int_{\As}  \pi'(\de a|s) \wass{p'(\cdot|s,a)}{p'(\cdot|\overline{s},a)} \le L_{p'} d_{\Ss}(s,\overline{s}).
	\end{align*}
	For the second integral, we need to compute the Lipschitz semi-norm of $l_f$. Let $(s,a),(\overline{s},\overline{a}) \in \SAs$:
	\begin{align*}
	l_f(s,a) - l_f(\overline{s},\overline{a}) & = \int_{\Ss}  \left( p'(\de s'|s,a) - p'(\de s'|\overline{s},\overline{a})\right) f(s')  \\
	&  \le \wass{p'(\cdot|s,a)}{p'(\cdot|\overline{s},\overline{a})} \le L_{p'} d_{\SAs}((s,a),(\overline{s},\overline{a})).
	\end{align*}
	By plugging this into the second integral we obtain:
	\begin{align*}
	\int_{\As}  \left( \pi'(\de a|s) -\pi'(\de a|\overline{s}) \right) l_f(s,a) \le L_{p'} \wass{\pi'(\cdot|s)}{\pi'(\cdot|\overline{s})} \le L_{p'}L_{\pi'} d_{\Ss}(s,\overline{s}).
	\end{align*}
\end{proof}

\begin{lemma}\label{lemma:decompAdv}
Let $\mathbb{A}_{\pi,p}^{\pi',p'}$ be the expected coupled relative advantage function, $\mathbb{A}_{\pi,p}^{\pi',p}$ and $ \mathbb{A}_{\pi,p}^{\pi,p'}$ be the expected (decoupled) policy and configuration relative advantage functions respectively. Then, it holds that:
	\begin{align*}
		\left| \mathbb{A}_{\pi,p}^{\pi',p'} - \left(\mathbb{A}_{\pi,p}^{\pi,p'} + \mathbb{A}_{\pi,p}^{\pi',p}\right) \right| \le \int_{\Ss} \mu_{\pi,p}(\de s) \wass{\pi'(\cdot|s)}{\pi(\cdot|s)} \lip{A_{\pi,p}^{\pi,p'}}.
	\end{align*}
\end{lemma}

\begin{proof}
	We can rewrite the expected relative advantage $\mathbb{A}_{{\pi,p}}^{\pi',p'}$ using Lemma~\ref{thr:lemmaAdvantage}:
	\begin{align}
		\mathbb{A}_{{\pi,p}}^{\pi',p'} & = \int_{\mathcal{S}} \piPD(\de s) A_{\pi,p}^{\pi',p'}(s)  \nonumber \\
		& = \int_{\mathcal{S}} \piPD(\de s) \left( A_{\pi,p}^{\pi',p}(s) + \int_{\mathcal{A}} \pi'(\de a|s) A_{\pi,p}^{\pi,p'}(s,a)  \right)  \label{line:b0} \\
		& = \int_{\mathcal{S}}  \piPD(\de s) A_{\pi,p}^{\pi',p}(s) + \int_{\mathcal{S}} \int_{\mathcal{A}}  \piPD(\de s) \pi(\de a|s) A_{\pi,p}^{\pi,p'}(s,a) \nonumber \\
		& \quad + \int_{\mathcal{S}} \piPD(\de s)  \int_{\mathcal{A}} \left( \pi'(\de a|s) - \pi(\de a|s) \right) A_{\pi,p}^{\pi,p'}(s,a) \nonumber \\
		& =  \mathbb{A}_{{\pi,p}}^{\pi',p} + \mathbb{A}_{{\pi,p}}^{\pi,p'} + \int_{\mathcal{S}} \piPD(\de s)  \int_{\mathcal{A}} \left( \pi'(\de a|s) - \pi(\de a|s) \right) A_{\pi,p}^{\pi,p'}(s,a), \label{line:b1}
	\end{align}
	where line~\eqref{line:b0} comes from Lemma~\ref{thr:lemmaAdvantage}.	From Equation~\eqref{line:b1} we can straightforwardly state the following inequalities:
	\begin{align*}
		\mathbb{A}_{{\pi,p}}^{\pi',p'} & \geq \mathbb{A}_{{\pi,p}}^{\pi',p} + \mathbb{A}_{{\pi,p}}^{\pi,p'}  - \bigg| \int_{\mathcal{S}} \piPD(\de s)  \int_{\mathcal{A}} \left( \pi'(\de a|s) - \pi(\de a|s) \right) A_{\pi,p}^{\pi,p'}(s,a)  \bigg|, \\
		\mathbb{A}_{{\pi,p}}^{\pi',p'} & \leq \mathbb{A}_{{\pi,p}}^{\pi',p} + \mathbb{A}_{{\pi,p}}^{\pi,p'}  + \bigg| \int_{\mathcal{S}} \piPD(\de s)  \int_{\mathcal{A}} \left( \pi'(\de a|s) - \pi(\de a|s) \right) A_{\pi,p}^{\pi,p'}(s,a) \bigg|.
	\end{align*}
	Then, we bound the absolute value in the right hand side:
	\begin{align}
		\Big| \mathbb{A}_{{\pi,p}}^{\pi',p'} - \left(\mathbb{A}_{{\pi,p}}^{\pi',p} + \mathbb{A}_{{\pi,p}}^{\pi,p'} \right) \Big| & \leq \bigg| \int_{\mathcal{S}} \piPD(\de s)  \int_{\mathcal{A}} \left( \pi'(\de a|s) - \pi(\de a|s) \right) A_{\pi,p}^{\pi,p'}(s,a) \bigg| \notag \\
		& \leq \int_{\mathcal{S}} \piPD(\de s)  \left| \int_{\mathcal{A}}  \pi'(\de a|s) - \pi(\de a|s)  A_{\pi,p}^{\pi,p'}(s,a) \right| \notag \\
		& \leq \int_{\mathcal{S}}  \piPD(\de s)  \wass{\pi'(\cdot|s) }{\pi(\cdot|s) } \sup_{s \in \Ss} \lip{A_{\pi,p}^{\pi,p'}(s,\cdot)} \label{p:x5100},
	\end{align}
	where line~\eqref{p:x5100} follows from the definition of Wasserstein distance. The result is obtained by observing that $\sup_{s \in \Ss} \lip{A_{\pi,p}^{\pi,p'}(s,\cdot)} \le \lip{A_{\pi,p}^{\pi,p'}}$.
\end{proof}

\begin{lemma}\label{lemma:lipNormAdv}
	Let $\pi,\pi' \in \PP$ and $p,p' \in \PPP$ be two policies and configurations for a Conf-MDP. Then, it holds that:
	\begin{align*}
		\lip{A_{\pi,p}^{\pi',p'}} \le \lip{A_{\pi,p}^{\pi',p}} + (L_{\pi'} + 1) \left\| A_{\pi,p}^{\pi,p'} \right\|_L .
	\end{align*}
\end{lemma}

\begin{proof}
	First of all, we exploit Lemma~\ref{thr:lemmaAdvantage} and	the fact that $\lip{\cdot}$ is a semi-norm and, thus, it satisfies the triangular inequality:
	\begin{align*}
		\lip{A_{\pi,p}^{\pi',p'}} = \lip{A_{\pi,p}^{\pi',p} + \int_{\As} \pi'(\de a|\cdot) A_{\pi,p}^{\pi,p'}(\cdot,a)} \le \lip{A_{\pi,p}^{\pi',p}} + \lip{\int_{\As} \pi'(\de a|\cdot) A_{\pi,p}^{\pi,p'}(\cdot,a)}. 
	\end{align*}
	For what concerns the second term, we proceed as follows for $s,\overline{s} \in \Ss$:
	\begin{align*}
		\int_{\As} \pi'(\de a|s) & A_{\pi,p}^{\pi,p'}(s,a)  - \int_{\As} \pi'(\de a|\overline{s}) A_{\pi,p}^{\pi,p'}(\overline{s},a) \\
		&  = \int_{\As} \pi'(\de a|s) A_{\pi,p}^{\pi,p'}(s,a) - \int_{\As} \pi'(\de a|\overline{s}) A_{\pi,p}^{\pi,p'}(\overline{s},a) \pm  \int_{\As} \pi'(\de a|\overline{s}) A_{\pi,p}^{\pi,p'}(s,a) \\
		& = \int_{\As} \left( \pi'(\de a|s) - \pi'(\de a|\overline{s})\right)  A_{\pi,p}^{\pi,p'}(s,a) + \int_{\As} \pi'(\de a|\overline{s}) \left( A_{\pi,p}^{\pi,p'}(s,a) - A_{\pi,p}^{\pi,p'}(\overline{s},a)\right) \\
		& \le \lip{A_{\pi,p}^{\pi,p'}(s,\cdot)} \wass{\pi'(\cdot|s)}{\pi(\cdot|s)} + \int_{\As} \pi'(\de a|\overline{s}) \lip{A_{\pi,p}^{\pi,p'}(\cdot,a)} d_{\Ss}(s,\overline{s}) \\
		& \le  L_{\pi'}\sup_{s \in \Ss} \lip{A_{\pi,p}^{\pi,p'}(s,\cdot)} d_{\Ss}(s,\overline{s}) + \sup_{a \in \As}  \lip{A_{\pi,p}^{\pi,p'}(\cdot,a)} d_{\Ss}(s,\overline{s}).
	\end{align*}
	Finally, we observe that $\left\| A_{\pi,p}^{\pi,p'} \right\|_L \ge \max\left\{ \sup_{a \in \As}  \lip{A_{\pi,p}^{\pi,p'}(\cdot,a)}, \sup_{s \in \Ss} \lip{A_{\pi,p}^{\pi,p'}(s,\cdot)} \right\}$.
\end{proof}

\end{document}